\documentclass{scrartcl}
\usepackage[ruled]{algorithm2e} 
\usepackage{graphicx}
\usepackage[round]{natbib}
\usepackage{subcaption}
\usepackage{amsmath}
\usepackage{amsfonts}
\usepackage{amssymb}
\usepackage{amsthm}
\usepackage{mathtools}
\usepackage{booktabs}
\usepackage{listings}[language=Julia]
\usepackage{wrapfig}
\usepackage{hyperref}

\newtheorem{dfn}{Definition}
\newtheorem{pbm}{Problem}
\newtheorem{thm}{Theorem}

\newtheorem{crl}{Corollary}

\title{Efficient Incremental Belief Updates\\Using\\Weighted Virtual Observations}

\author{David Tolpin,\\Offtopia and Ben-Gurion University of the Negev\\ david.tolpin@gmail.com}

\begin{document}

\maketitle
\begin{abstract}
	We present an algorithmic solution to the problem of
	incremental belief updating in the context of Monte Carlo
	inference in Bayesian statistical models represented by
	probabilistic programs.  Given a model and an inferred
	posterior, approximated by samples, our solution constructs
	a set of weighted observations to condition the model such that
	inference would result in the same posterior. This problem
	arises, for example, in multi-level modelling, incremental
	inference, inference in presence of privacy constraints.
	According to the solution, first, a set of virtual
	observations is selected, then, observation weights are found
	through a computationally efficient optimization procedure
	such that the posterior resulting from the weighted virtual
	observations coincides with or closely approximates the
	original posterior.  We implement and apply the solution to a
	number of didactic examples and case studies, showing
	efficiency and robustness of our approach.  The provided
	reference implementation is agnostic to the probabilistic
	programming language and inference algorithm used, and can
	be straightforwardly applied to most mainstream
	probabilistic programming environments.
\end{abstract}

\section{Introduction}

Probabilistic programming aims at enabling tractable inference in
elaborate Bayesian statistical models, represented as programs.
Generality of model structure makes Monte Carlo inference
algorithms a popular choice~\citep{WVM14}. However, approximating the
posterior by a set of Monte Carlo samples breaks
incremental belief updating. Indeed, a probabilistic program
can be viewed as a function $p(x, y)$ computing the joint probability of
the latent variables $x$ and the observables $y$, and is often
factorized as $p(x)p(y|x)$. $p(x, y)$, along with observations
$\pmb{y}^*$, can be passed to an inference algorithm.
However, the posterior $p(x, y|\pmb{y}^*)=p(x|\pmb{y}^*)p(y|x)$ does not
have a similar convenient functional form; instead, $p(x|\pmb{y}^*)$ is
approximated by a set of Monte Carlo samples.

The need for a representation of the posterior suitable for
further belief updates often arises in applications. 
\begin{itemize}
\item In an online setting, such as stock price monitoring~\citep{RSM+20},
observations arrive gradually over time. The posterior belief
about latent variables should be updated as the observations
arrive, and all past observations should affect the up-to-date
posterior belief, without reprocessing the whole history upon
arrival of each new piece of evidence.
\item Distributed autonomous systems, such as multi-robot
environments~\citep{GSP22}, involve communication among parties.
Large amounts of raw information would put burden on both
communication channels and computational capacities of each
party. Instead, the parties can share their posterior beliefs
about latent parameters, and use other parties' beliefs for
inference.
\item Privacy considerations, e.g. in health care~\citep{XC21},
may preclude sharing observations obtained in one entity for
inference over parameters of another entity.  Instead of sharing
observations, posterior beliefs about latent parameters can be
shared. 
\end{itemize}
One widespread way to address this need is \textit{empirical
Bayes}~\citep{R51,C85,R92}. Empirical Bayes approximates the
inferred posterior by a parametric distribution fit to the set
of Monte Carlo samples of $x$, which serves as
the prior for the next belief update. However, empirical
Bayes requires guessing or selecting an appropriate parametric
form for approximation of the posterior. Except in trivial
cases, parametric approximation of the posterior introduces
inaccuracies.  Although efficiency of empirical Bayes was
demonstrated in a number of settings~\citep{R92}, in general,
empirical Bayes may result in misspecified models and
biased inference outcomes.

Instead of looking for a parametric approximation of the
posterior, we propose to condition the model (appropriately
modified) on a small set of weighted virtual observations, such
that the inference on this set results in the same posterior as
the inference on the original set of observations, but at a much
lower computational cost.  This is related to Bayesian
coresets~\citep{HCB16,ZKK+21}, although with essential
differences, discussed in Section~\ref{sec:related}.  The
weighted virtual observations are obtained algorithmically using
an efficient optimization procedure in a way that is agnostic
either to the structure, representation of the statistical model or to the
inference algorithm.

This work brings the following contributions:
\begin{enumerate}
\item Formulation and theoretical analysis of the problem of
finding the observation set given the model and the posterior.

\item A theoretically justified algorithm for constructing a
weighted virtual observation set.

\item A reference implementation of the solution, applied to a
range of didactic examples and case studies.
\end{enumerate}

We proceed as follows. Section~\ref{sec:problem} formalizes the problem
and analyses its hardness in the general and special cases.
Section~\ref{sec:solution} presents our solution to the problem
of efficient incremental belief updates through weighted virtual
observations, the main contribution of this work.
Section~\ref{sec:studies} evaluates the solution on a range of case
studies. Section~\ref{sec:related} discusses related work, and
Section~\ref{sec:conclusion} concludes the paper. Algorithm
implementations in the Julia programming language~\citep{BEK+17} and Julia
Pluto notebooks for examples and case studies are at
\url{https://bitbucket.org/dtolpin/wvo-supp/}.

\section{Problem}
\label{sec:problem}

We now state the problem formally. We start with the general
problem, and show that the general problem is hard. We then
identify a special case in which the problem admits an efficient
solution. Based on these observations, we state a modified
problem, for which a computationally feasible solution is
derived and evaluated in the rest of the paper.

\subsection{General Problem}

\begin{pbm}
\label{pbm:general} A statistical model is defined
by its join probability (mass or density) function $p(x,
y)$, where $x$ is the latent variable, and $y$ is the
observable variable (both can be vectors). The model is
conditioned on a set of observations, and a set
$\pmb{x}^*$ of $S$ i.i.d. samples from the posterior is
obtained.

Given $p(x, y)$ and $\pmb{x}^*$, what is the
most likely set of observations on which the model was
conditioned?
\end{pbm}

\subsection{Hardness of General Case}

To justify our quest for an efficient approximate solution for
Problem~\ref{pbm:general}, we first show that reconstructing
observations from the model and the posterior is computationally
hard.

\begin{thm}
\label{thm:nphard}
The problem of observation reconstruction is NP-hard. 
\end{thm}
\begin{proof}
We prove by reduction from SAT. Let $\Phi(\chi_1, \dots, \chi_n)$ be a
boolean formula. We construct a model, in which
\begin{itemize}
\item $x=\chi_1$ is the latent variable,
\item $y=\chi_2, \dots, \chi_n$ is the observable variable,
\item the prior distribution of $x$ is $p(x=T)=p(x=F)=0.5$,
\item the conditional distribution of $y$ given $x$ is $p(y|x)
      \propto 1$ if $\Phi(\cdot)$ is satisfied, $0$ otherwise.
\end{itemize}
To solve the SAT problem, we use three queries to the observation
reconstruction algorithm:
\begin{equation}
\begin{aligned}
p(x=T|y)& = 1 \\
p(x=F|y)& = 1 \\
p(x=T|y)& = 0.5
\end{aligned}
\end{equation}
In words, the first query checks whether there is a satisfying
assignment to $\chi_2, \dots, \chi_n$ with $\chi_1=T$, the second
query --- with $\chi_1=F$, the third query --- with both
$\chi_1=T$ and $\chi_1=F$.  If at least one of the queries
returns a solution, the formula is satisfiable.  Otherwise, it
is not. A satisfying assignment can be recovered from the
reconstructed observation.
\end{proof}

\subsection{Solution for Conjugate Models}
		
Having shown that the problem is hard in general, we now
identify a class of models for which the problem admits an
efficient, closed-form solution. Conjugate
models~\citep{GCS+13,K36} constitute such a class, as established
by Theorem~\ref{thm:conjugate}.

\begin{thm}
\label{thm:conjugate}
Let the joint density $p(x, y)$  be defined by a conjugate
exponential family model, as follows:
\begin{equation}
\begin{aligned}
p(x, y) & = p(x)p(y|x) \\
p(y|x) & \propto g(x)\exp(\phi(x)^\top t(y)) \\
p(x) & \propto g(x)^\eta\exp\left(\phi(x)^\top \nu\right) 
\end{aligned}
\end{equation}
Let, then, the parameter posterior $x|\pmb{y}^*$ resulting from conditioning
on a set $\pmb{y}^*$ of $N^*$ observations $\pmb{y}^*$ be represented by a set
of $S$ samples $\pmb{x}^*$.

Then, any set $\hat{\pmb{y}}$ of $\hat N$ observations satisfying 
\begin{equation}
\sum_{i=1}^{\hat N} t(\hat y_i)=\tilde \nu^* - \nu
\end{equation}
where $\tilde \nu^*$ is the maximum likelihood estimate of the
parameter $\nu^*$ of the posterior of the form 
\begin{equation}
p(x|\pmb{y*}) = \propto g(x)^{\eta + N}\exp(\left(\phi(x)^\top \nu^*\right)
\end{equation}
on the sample set $x^*$, will yield a posterior distribution that approaches
the posterior as $S \to \infty$.
\end{thm}
\begin{proof}
Due to conjugacy, the posterior has the form~\citep{GCS+13}
\begin{equation}
p(x|\pmb{y}^*) \propto g(x)^{\eta + N}\exp\left(\phi(x)^\top (\nu + \sum_{i=1}^{N^*} t(y^*_i))\right).
\end{equation} 
that is, $\nu^* = \nu + \sum{i=1}^{N^*} t(y_i^*)$. Since the
posterior in conjugate models is determined by the number of
observations and their sufficient statistics, any set of $\hat N$
observations with the same total sufficient statistics 
\begin{equation}
\sum_{i=1}^{\hat N} t(\hat y_i) = \sum_{i=1}^{N^*} t(y_i^*) = \nu^* - \nu
\label{eqn:conjugate-suff-stat}
\end{equation}
yields the same posterior. Maximum likelihood estimator is
consistent~\citep[Chapter 7]{CB01}, hence $\tilde \nu^* \to \nu^*$ as $S \to
\infty$, which completes the proof.
\end{proof}

As an illustration, let us apply the computation implied by
Theorem~\ref{thm:conjugate} to the Beta-Bernoulli model:

\begin{equation}
\begin{aligned}
p(y|x) & = \mathrm{Bernoulli}(y; p=x) \\
p(x) & = \mathrm{Beta}(x; \alpha, \beta)
\end{aligned}
\end{equation}

Observations $y$ take values of either $0$ or $1$, the summary
statistic of the Bernoulli distribution is just the value of
$y$, $t(y) \coloneqq y$, and the canonical exponential family form of
the density of the Beta distribution is
\begin{equation}
\begin{aligned}
\mathrm{Beta}(x; \alpha, \beta) & = \frac 1 {\mathrm{B}(\alpha, \beta)} x^{\alpha}(1-x)^{\beta} \\
& = \frac 1 {\mathrm{B}(\alpha, \beta)} (1 - x)^{\alpha + \beta} \exp \left( \alpha
\log \frac x {1 - x}\right)
\end{aligned}
\end{equation}
with the correspondence $g(x) \coloneqq 1 - x$, $\eta \coloneqq \alpha +
\beta$, $\nu \coloneqq \alpha$. 

Consider the case that for $\alpha=1, \beta=1$ and $N^*=10$
observations we infer the posterior approximated by $S=5000$
samples,  such that the maximum likelihood estimate of
$\alpha^*$ is $\alpha^* = 7$\footnote {To simplify the example,
we assume that $\alpha$ is rounded to the nearest integer.}.
Solving~\eqref{eqn:conjugate-suff-stat}, we obtain that
conditioning on any set of $10$ observations with $6$ successes
($1$) and $4$ failures ($0$) yields the same posterior.

Theorem~\ref{thm:conjugate} does not bear practical
significance: indeed, for conjugate models we can directly
obtain the parameters of the posterior and plug them into the
model for incremental belief updating. However, the theorem
suggests that efficient approximate solutions for the problem of
reconstructing observations given the posterior can be found.

\subsection{Problem Admitting Efficient Approximate Solution}

Having shown that the general problem is hard, we seek to define
a modified problem that can be solved efficiently, exactly or
approximately. To this end, we want to address three aspects of
the problem definition.

\begin{itemize}
\item First, optimization is efficient in a differentiable
parameter space, however we do not want to restrict the problem
to a class of models where observations belong to
$\mathcal{R}^n$, and where the joint probability is
differentiable by observations. Instead of searching for
observations themselves, we can choose and fix a candidate set
of reconstructed observations, for example by sampling from the
predictive posterior, and search for observation
\textit{weights} such that the sum of the weights is equal to
the size $N^*$ of the original observation set.  Conditioning on a
weighted set of observations is related to distributional or
uncertain evidence~\citep{TZR+21,MMW23}, and
can be seen as searching for the most likely categorical
distribution on the reconstructed observation set.

\item Second, to make the problem easier to solve, we may reveal
more information about the inference; in particular, we may
reveal the original observation set $\pmb{y}^*$ along with 
the inferred posterior $\pmb{x}^*$, rather than just the number
of observations. This relaxation makes practical sense ---
indeed, the original observations are available at the inference
time anyway --- and, as we will see later, helps to devise an
efficient solution.

\item Third, our motivation is to find a set of observations that
results in the same or approximately the same posterior. Rather
than looking for the most likely set of observations given the
posterior, which implies inversion of the model's dependencies
($y$ would have to be conditioned on $\pmb{x}^*$) we should look
for a set of observations that minimizes the divergence from the
original posterior, e.g. in terms of KL divergence.
\end{itemize}

Following these three ideas, we formulate a modified problem,
which, as we will later see, admits an efficient solution.
\begin{pbm}
\label{pbm:single-level}
A statistical model is defined by its factored joint probability (mass or
density) function $p(x, y)=p(x)p(y|x)$, where $x$ is the latent variable,
and $y$ is the observable variable (both can be vectors). The
model is conditioned on a set $\pmb{y}^*$ of $N^*$
observations, and a set $\pmb{x}^*$ of $S$ i.i.d. samples from
the posterior is obtained.

Given $\pmb{x}^*$, $\pmb{y}^*$, and a set of virtual
observations $\hat{\pmb{y}}$ of size $\hat N$, what should be the 
weights $\pmb{w} \in \mathcal{R}^{\hat N}$ of the virtual
observations, so that KL divergence between the original
posterior $x|\pmb{y}^*$ and the reconstructed posterior
$x|\hat{\pmb{y}}, \pmb{w}$ is minimized?
\end{pbm}

Problem~\ref{pbm:single-level} considers a flat,
single-level model, in which virtual observations are sought
with respect to the originally observed random variables.
However, a real power of virtual observations is revealed in
multi-level models (see Section~\ref{sec:multi-level-models}):

\begin{pbm}
\label{pbm:multi-level}
A multi-level statistical model is defined by its factored joint
probability (mass or density) function $p(x, z_k, y_k) = (x)p(z_k|x)p(y_k|z_k),\,
k \in 1 \dots K^*$, where
$x$ and $z_k$ are the latent variables, and $y_k$ are the observable
variables; $x$ is customarily called the 
\textit{hyperparameter}, and $z_k$ are the \textit{group
parameters}. The model is conditioned on K sets $\pmb{y_k}^*$ of $N_k^*$
observations, and a set $\pmb{x}^*$ of $S$ i.i.d.
samples from the posterior is obtained.

Given $\pmb{x}^*$, $\pmb{y}^*_1, \dots, \pmb{y}^*_{K^*}$, and $\hat K$ sets of virtual
observations $\hat{\pmb{z}}_k$ of size $\hat M$,
what should be the 
weights $\pmb{v} \in \mathcal{R}^{\hat K}$ of observation groups and
the weights $\pmb{w}_k \in \mathcal{R}^{M}$ of observations
within each group so that the KL divergence between the original
posterior $x|\pmb{y}^*$ and the reconstructed posterior
$x|\hat{\pmb{z}}, \pmb{v}, \pmb{w}_1, \dots, \pmb{w}_{\hat K}$ is minimized?
\end{pbm}

\section{Solution}
\label{sec:solution}

In this section, we gradually develop a solution to the problem
of replacing original observations with weighted virtual
observations.  Our ultimate goal and final
result is a solution for Problem~\ref{pbm:multi-level}.
However, we start with the basic cases of a flat, single-level
model (Problem~\ref{pbm:single-level}) and a two-level
model with a single group. These two basic cases provide us with
building blocks for solving the problem for multi-level models.

\subsection{Single-Level Models}

The probability of $\pmb{y}^*$ given $x$ has the form
\begin{equation}
p(\pmb{y}^*|x) = \prod_{i=1}^{N^*} p(y^*_i|x) = \exp \left(\sum_{i=1}^{N^*} \log p(y^*_i|x) \right)
\label{eqn:p-y-star-given-x-single-level}
\end{equation}

\begin{dfn}
\label{dfn:p-w-given-x-single-level}
Following the form of \eqref{eqn:p-y-star-given-x-single-level}, we define the probability
of $\pmb{w}$ given $\pmb{x}^*$ for Problem~\ref{pbm:single-level} as
\begin{equation}
p(\pmb{w}|x) = h(\pmb{w}) \prod_{i=1}^{\hat N} p(\hat y_i|x)^{w_i} = \exp\left( \log h(\pmb{w}) + \sum_{i=1}^{\hat N} w_i \log p(\hat y_i | x) \right)
\label{eqn:p-w-given-x-single-level}
\end{equation}
where $h(\pmb{w})$ is a base measure over $\pmb{w}$.
\end{dfn}

Definition~\ref{dfn:p-w-given-x-single-level} can be interpreted as the
conditional probability of a multinomial distribution~\citep{TZR+21,MMW23}
with the number of trials $W=\sum_{i=1}^{\hat N} w_i$ and event
probabilities $\frac {w_i} W$ over events $\hat y_i$ for $i \in
1 \dots \hat N$.

A guide to algorithmic solution of Problem~\ref{pbm:single-level}
is given by Theorem~\ref{thm:single-level}.
\begin{thm}
\label{thm:single-level}
Under the settings of Problem~\ref{pbm:single-level}, $\pmb{w}$
is Monte Carlo approximated by a solution of
\eqref{eqn:single-level}:
\begin{equation}
\begin{aligned}
\pmb{w} \approx & \arg \max \frac 1 S \sum_{s=1}^S   \sum_{i=1}^{\hat N} w_i \log p(\hat y_i|x_s^*)  \\
       &- \log \sum_{s=1}^S \exp\left(\sum_{i=1}^{\hat N}  w_i \log p(\hat y_i|x_s^*) - \sum_{i=1}^{N^*} \log p(y_i^*|x_s^*)\right) \\
&\mbox{s.t. } \sum_{i=1}^{\hat N} w_i = N^*
\end{aligned}
\label{eqn:single-level}
\end{equation}
The approximation tends to the exact solution as $S \to \infty$.
\end{thm}

This optimization objective has an intuitive interpretation:
\begin{itemize}
\item The first term maximizes the average log
	probability of the weighted set of observations $\hat{\pmb{y}},
	\pmb{w}$ given $\pmb{x}^*$.
\item The second term penalizes the probability of the weighted
	set of observations $\hat{\pmb{y}}, \pmb{w}$ for being greater
	than the conditional probability of the original set of
	observations $\pmb{y}^*$, for any $x^*$.
\end{itemize}
Maximizing the first term only would put all of the weight on a
single $\hat y_i$, which would yield a posterior with a lower
variance than the original one. Minimizing the second term only
would yield a posterior with a higher variance than the original
one. 

\begin{proof}
We write our optimization target in terms of KL divergence between
$\pmb{x}|\pmb{y}^*$ and $\pmb{x}|\hat{\pmb{y}}, \pmb{w}$ as
\begin{equation}
\begin{aligned}
\pmb{w} & = \arg \min \mathrm{KL}(p(x|\pmb{y}^*)||p(x|\pmb{w}))  \\
        & = \arg \max \mathbb{E}_{x|\pmb{y}^*}  \log \frac {p(x|\pmb{w})} {p(x|\pmb{y}^*)} \\
        & = \arg \max \mathbb{E}_{x|\pmb{y}^*}  \log \frac {p(\pmb{w}|x) p(x)} {p(\pmb{w}) p(x|\pmb{y}^*)} \\
&\mbox{s.t. } \sum_{i=1}^{\hat N} w_i = N^*.
\end{aligned}
\end{equation}

For optimization, we only need to keep the terms that depend on $\pmb{w}$:

\begin{equation}
\begin{aligned}
\pmb{w} & = \arg \max \mathbb{E}_{x|\pmb{y}^*} \log \frac {p(\pmb{w}|x)} {p(\pmb{w})} \\
        & =\arg \max \mathbb{E}_{x|\pmb{y}^*} \log {p(\pmb{w}|x)}  - \log p(\pmb{w}) \\
        & =\arg \max \mathbb{E}_{x|\pmb{y}^*} \sum_{i=1}^{\hat N} w_i \log p(\hat y_i|x) + \log h(\pmb{w}) - \log p(\pmb{w})
\end{aligned}
\end{equation}

We now need to rewrite $\log p(\pmb{w})- \log h(\pmb{w})$ in terms of
$p(\pmb{w}|x)$ and $p(\pmb{y}^*|x)$, which we can compute:
\begin{equation}
\begin{aligned}
\log p(w) & = \log \mathbb{E}_{x|\pmb{y}^*} p(\pmb{w}|x) \frac {p(x)} {p(x|\pmb{y}^*)} \\
\mbox{By }&\mbox{application of the Bayes' rule,} \\
& = \log \mathbb{E}_{x|\pmb{y}^*} p(\pmb{w}|x) \frac {p(\pmb{y}^*)} {p(\pmb{y}^*|x)} \\
& = \log \mathbb{E}_{x|\pmb{y}^*}  \frac {p(\pmb{w}|x)} {p(\pmb{y}^*|x)} + \log p(\pmb{y}^*) \\
& = \log \mathbb{E}_{x|\pmb{y}^*} \frac {\prod_{i=1}^{\hat N} p(\hat y_i|x)^{w_i}} {\prod_{i=1}^{N^*} p(y^*_i|x)} + \log h(\pmb{w}) + \log p(\pmb{y}^*) \\
& = \log \mathbb{E}_{x|\pmb{y}^*} \exp \left( \sum_{i=1}^{\hat N} w_i \log p(\hat y_i|x) - \sum_{i=1}^{N^*} \log p(y^*_i|x) \right)  + \log h(\pmb{w}) + \log p(\pmb{y}^*) 
\end{aligned}
\end{equation}

We do not care about $h(\pmb{w})$ because it cancels out in $\log p(\pmb{w}) - h(\pmb{w})$. For optimization, we can drop the constant term $p(\pmb{y}^*)$ and obtain:

\begin{equation}
\begin{aligned}
\pmb{w} =& \arg \max \mathbb{E}_{x|\pmb{y}^*} \sum_{i=1}^{\hat N} w_i \log p(\hat y_i|x) \\
&- \log \mathbb{E}_{x|\pmb{y}^*} \exp \left( \sum_{i=1}^{\hat N} w_i \log p(\hat y_i|x)  - \sum_{i=1}^{N^*} \log p(y_i^*|x)\right)
\end{aligned}
\end{equation}

We approximate the expectations by Monte Carlo integration over the posterior samples to finally obtain:

\begin{equation}
\begin{aligned}
\pmb{w} \approx &\arg \max \frac 1 S \sum_{s=1}^S   \sum_{i=1}^{\hat N} w_i \log p(\hat y_i|x_s^*)  \\
       &- \log \sum_{s=1}^S \exp\left(\sum_{i=1}^{\hat N}  w_i \log p(\hat y_i|x_s^*) - \sum_{i=1}^{N^*} \log p(y_i^*|x_s^*)\right) \\
\end{aligned}
\end{equation}
By the virtue of consistency of Monte Carlo approximation, the
approximation tends to the exact solution as $S$ tends to
infinity.
\end{proof}

Theorem~\ref{thm:single-level} hints at an algorithm
for finding $\pmb{w}$, as follows:
\begin{algorithm}
\caption{Single-level problem.}
\label{alg:single-level}
\begin{enumerate}
\item Perform inference on the model of
Problem~\ref{pbm:single-level} given $\pmb{y}^*$ to
obtain $\pmb{x}^*$.
\item Guess $\hat{\pmb{y}}$, e.g. through sampling from the
inferred posterior predictive distribution.
\item Find $\pmb{w}$ by
maximizing~\eqref{eqn:single-level}.
\end{enumerate}
\end{algorithm}

Let us reify Algorithm~\ref{alg:single-level} on
examples.

\paragraph{Example: Beta-Bernoulli Model}

Model~\ref{model:betabern} represent a Bernoulli distribution with
a uniform $[0, 1]$ prior on the success probability:
\begin{equation}
	\begin{aligned}
	\theta & \sim \mathrm{Beta}(1, 1) \\
	y & \sim \mathrm{Bernoulli}(\theta)
	\end{aligned}
	\label{model:betabern}
\end{equation}
The latent variable of the model is $\theta$. First, we infer
the posterior distribution of $\theta$ given 12 observations
$\pmb{y}^*=\{1, 1, 1, 0, 1, 0, 0, 1, 1, 1, 0, 1\}$.
The posterior is shown in blue in
Figure~\ref{fig:betabern-wvo-posterior}. 
\begin{figure}
	\centering
	\begin{subfigure}[c]{0.29\linewidth} 
		\centering
		\includegraphics[width=0.9 \linewidth]{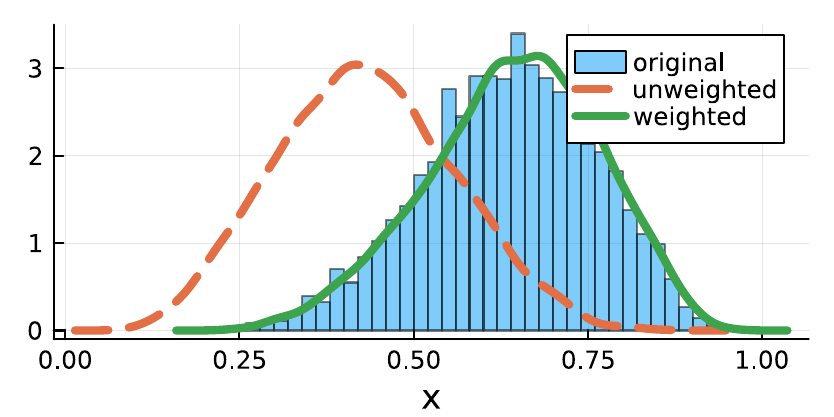}
		\caption{Posteriors} 
		\label{fig:betabern-wvo-posterior}
	\end{subfigure} 
	\begin{subfigure}[c]{0.29\linewidth} 
		\centering
		\includegraphics[width=0.9 \linewidth]{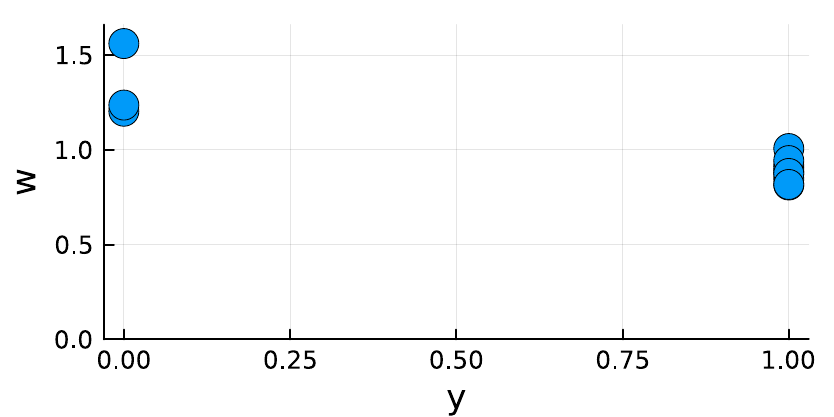}
		\caption{Inferred weights}
		\label{fig:betabern-weights} 
	\end{subfigure} 
	\caption{Beta-Bernoulli model: posteriors and weights.} 
\end{figure}
Second, we draw  12 samples $\hat{\pmb{y}}$ from the posterior
predictive distribution.  The parameter posterior conditioned on
$\hat{\pmb{y}}$ is likely to be quite different. For
$\hat{\pmb{y}} = \{0, 1, 1, 0, 0, 1, 0, 1, 0, 0, 0, 1\}$, the
posterior is shown in dashed red. Finally, we infer the weights
maximizing~\eqref{eqn:single-level}. The inferred weights are
shown in Figure~\ref{fig:betabern-weights}. For
Model~\ref{model:betabern}, it is easy to guess right weights,
which in our case are $\frac 4 7$ for each $\hat y_i = 0$ and
$\frac 8 5$ for each $\hat y_i = 1$ (or any other combination of
weights such that $\sum_{i=1}^{12} w_i\hat y_i =
8,\,\sum_{i=1}^{12} w_i(1-\hat y_i) = 4$). For empirical assessment,
we condition Model~\eqref{model:betabern} on the weighted
virtual observations $\hat{\pmb{y}}, \pmb{w}$. The inferred
posterior is shown in solid green and virtually coincides with
the original posterior.

\paragraph{Example: Normal Model with Non-Informative Prior}

Model~\ref{model:normal} represents a single-dimensional normal distribution
with unknown parameters:
\begin{equation}
	\begin{aligned}
		p(\mu, \log \sigma) & \propto 1 \\
		y|\mu, \sigma & \sim \mathrm{Normal}(\mu, \sigma)
	\end{aligned}
	\label{model:normal}
\end{equation}
The latent variables of the model are $\mu$ and $\sigma$.
First, we infer the posterior distributions of $\mu$ and
$\sigma$ given 10 observations $\pmb{y}^*=\{-0.56,$ $0.81,$
$-0.40,$ $1.10,$ $-0.8,$ $-0.35,$ $-1.65,$ $-0.37,$ $-0.20,$
$2.49\}$. The posterior is shown in blue in
Figure~\ref{fig:normal-wvo-posterior}. 
\begin{figure}
	\begin{subfigure}[c]{0.69\linewidth} 
		\centering
		\includegraphics[width=0.9 \linewidth]{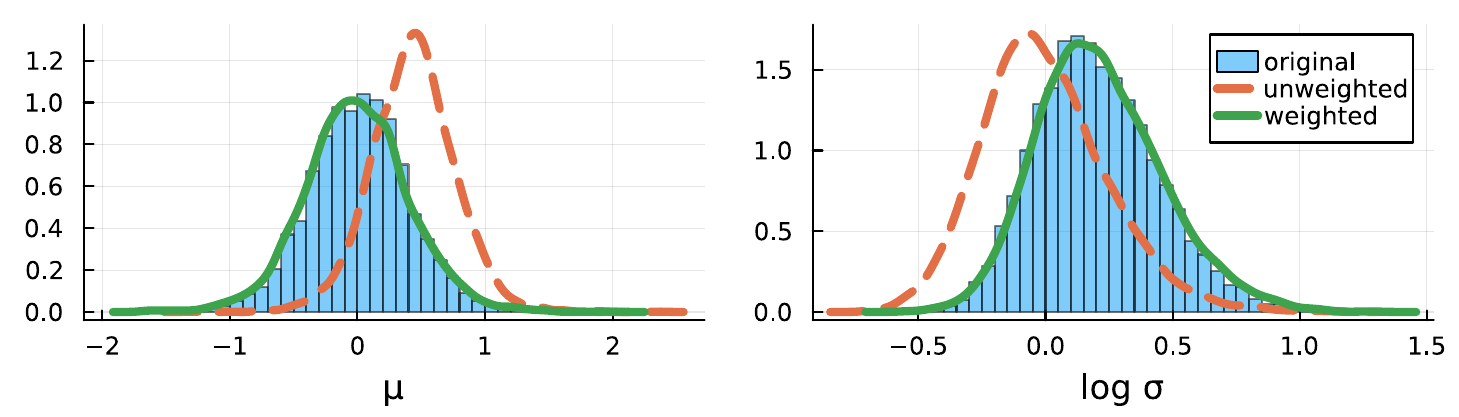}
		\caption{Posteriors} 
		\label{fig:normal-wvo-posterior}
	\end{subfigure} 
	\begin{subfigure}[c]{0.29\linewidth} 
		\centering
		\includegraphics[width=0.9 \linewidth]{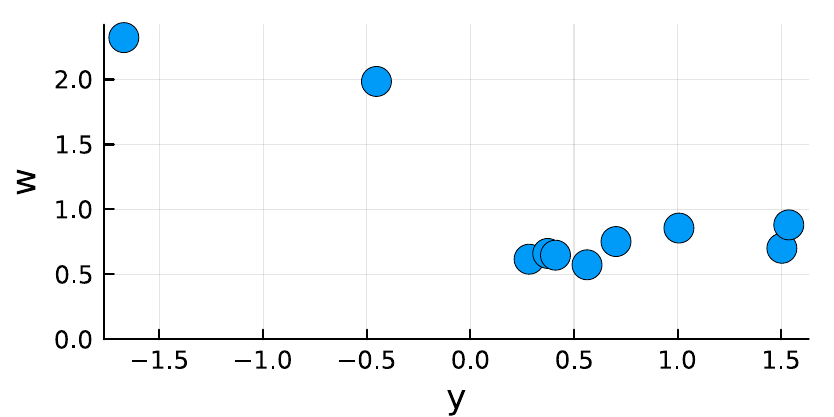}
		\caption{Inferred weights}
		\label{fig:normal-weights} 
	\end{subfigure} 
	\caption{Normal model: posteriors and weights.}
\end{figure}
Second, we draw  10 samples $\hat{\pmb{y}}$ from the posterior
predictive distribution.  The parameter posterior conditioned on
$\hat{\pmb{y}}$ is likely to be quite different. For
$\hat{\pmb{y}} = \{-1.67,$ $-0.45,$ $0.28,$ $0.37,$ $0.41,$
$0.56,$ $0.70,$ $1.01,$ $1.50,$ $1.54\}$, the posterior is shown
in dashed red. Finally, we infer the weights
maximizing~\eqref{eqn:single-level}. The inferred weights are
shown in Figure~\ref{fig:normal-weights}.  For assessment, we
condition Model~\eqref{model:normal} on the weighted virtual
observations $\hat{\pmb{y}}, \pmb{w}$. The inferred posterior is
shown in solid green and virtually coincides with the original
posterior.

\subsection{Multi-Level Models}

As a first step towards a solution of
Problem~\ref{pbm:multi-level}, let us consider a two-level
model with a single `intermediate' latent variable:
\begin{equation}
p(x, z, y) = p(x)p(z|x)p(y|z)
\end{equation}
As before, the model is conditioned on a set of $N^*$ observations
$\pmb{y}^*$, and the posterior is approximated by a set of $S$
samples $\pmb{x}^*$. Given a set of $\hat M$ virtual observations
$\hat{\pmb{z}}$ of $z$, we search for the weights $\pmb{w}$
that minimize the KL divergence between the original posterior
$x|\pmb{y}^*$ and the reconstructed posterior $x|\hat{\pmb{z}},
\pmb{w}$. 

The probability of $\pmb{y}^*$ given $x$ has the form
\begin{equation}
\begin{aligned}
p(\pmb{y}^*|x) &= \int_{\mathrm{dom}(z)} p(\pmb{y}^*|z)p(z|x) dz = \int_{\mathrm{dom}(z)} p(z|x) \prod_{i=1}^{N^*} p(y^*_i|z) dz \\
  &\approx \sum_{t=1}^T p(\tilde z_t|x) \frac {\prod_{i=1}^{N^*}p(y^*_i|\tilde z_t)} {p(\tilde z_t)}
\end{aligned}
\label{eqn:p-y-star-given-x-multi-level-K=1}
\end{equation}
for some set of i.i.d. samples $\tilde{\pmb{z}}$ from $\mathrm{dom}(z)$.
\begin{dfn}
\label{dfn:p-w-given-x-multi-level-K=1}
Following the form of \eqref{eqn:p-y-star-given-x-multi-level-K=1}, we define the probability
of $\pmb{w}$ given $\pmb{x}^*$ for Problem~\ref{pbm:multi-level} with $K=1$ as
\begin{equation}
p(\pmb{w}|x) = h(\pmb{w}) \sum_{i=1}^{\hat M} w_i p(\hat z_i|x)
\label{eqn:p-w-given-x-multi-level-K=1}
\end{equation}
where $h(\pmb{w})$ is a base measure over $\pmb{w}$.
\end{dfn}

Definition~\ref{dfn:p-w-given-x-multi-level-K=1} can be interpreted as the
conditional probability of an uncertainly observed
variable~\citep{J90,P88}  with observation probabilities $\frac
{w_i} W$ over observations $\hat z_i$ for $i \in 1 \dots \hat M$.

\begin{thm}
\label{thm:multi-level-K=1}
Under the settings of Problem~\ref{pbm:multi-level} with $K=1$,
$\pmb{w}$ is Monte Carlo approximated by a solution of
\eqref{eqn:multi-level-K=1}:
\begin{equation}
\begin{aligned}
\pmb{w} \approx & \arg \max \frac 1 S \sum_{s=1}^S   \log \sum_{i=1}^{\hat M} w_i p(\hat z_i|x_s^*)  \\
       &- \log \sum_{s=1}^S \exp\left(\log \sum_{i=1}^{\hat M}  w_i p(\hat z_i|x_s^*) - \sum_{i=1}^{N^*} \log p(y_i^*|x_s^*)\right) \\
&\mbox{s.t. } \sum_{i=1}^{\hat M} w_i = 1
\end{aligned}
\label{eqn:multi-level-K=1}
\end{equation}
The approximation tends to the exact solution as $S \to \infty$.
\end{thm}

\begin{proof} The proof is identical to the proof of
Theorem~\ref{thm:single-level}
with \eqref{eqn:p-w-given-x-multi-level} substituted instead of
\eqref{eqn:p-w-given-x-single-level}.
\end{proof}

Optimization target~\eqref{eqn:multi-level-K=1} still
depends on $p(\pmb{y}^*|x) = \int_{\mathrm{dom}(z)}
p(\pmb{y}^*|z)p(z|x) dz$.  Although a closed-form solution of the
integral is in general not available, for practical purposes
$p(\pmb{y}^*|x)$ can be estimated by Monte Carlo integration
over $z$, through forward sampling of $z|x$ for each $x \in
\pmb{x}^*$.  This computation does not depend on $\pmb{w}$ and
hence only has to be performed once, before the optimization of
$\pmb{w}$. With this addition, an algorithm for finding $\pmb{w}$ 
for Problem~\ref{pbm:single-level} also applies 
\textit{mutatis mutandis} to
Problem~\ref{pbm:multi-level} with $K=1$.

\paragraph{Example: Normal Model With Hyperprior on The Mean} 
Model~\eqref{eqn:hypernormal} is an extension of
Model~\eqref{model:normal} with a hyperprior imposed on the
mean:
\begin{equation}
\begin{aligned}
\nu, \log \tau & \sim \mathrm{Normal}(0, 1) \times \mathrm{Normal}(0, 1) \\
\mu, \log \sigma & \sim \mathrm{Normal}(\nu, \tau) \times \mathrm{Uniform}(-\infty, \infty) \\
y & \sim \mathrm{Normal}(\mu, \sigma)
\end{aligned}
\label{eqn:hypernormal}
\end{equation}
The hyperparameters of the model are $\nu$ and $\tau$, the
latent parameters are $\nu$ and $\sigma$, and the observed
variable is $y$. Unlike in Model~\eqref{model:normal}, the prior
on $\nu$ and $\log \tau$ has to be proper in order for the
posterior to be proper. The model is conditioned on the same set of
observations as Model~\eqref{model:normal} above. The posterior
is shown in blue in Figure~\ref{fig:hypernormal-wvo-posterior}.
\begin{figure}
	\centering
	\includegraphics[width=0.63\linewidth]{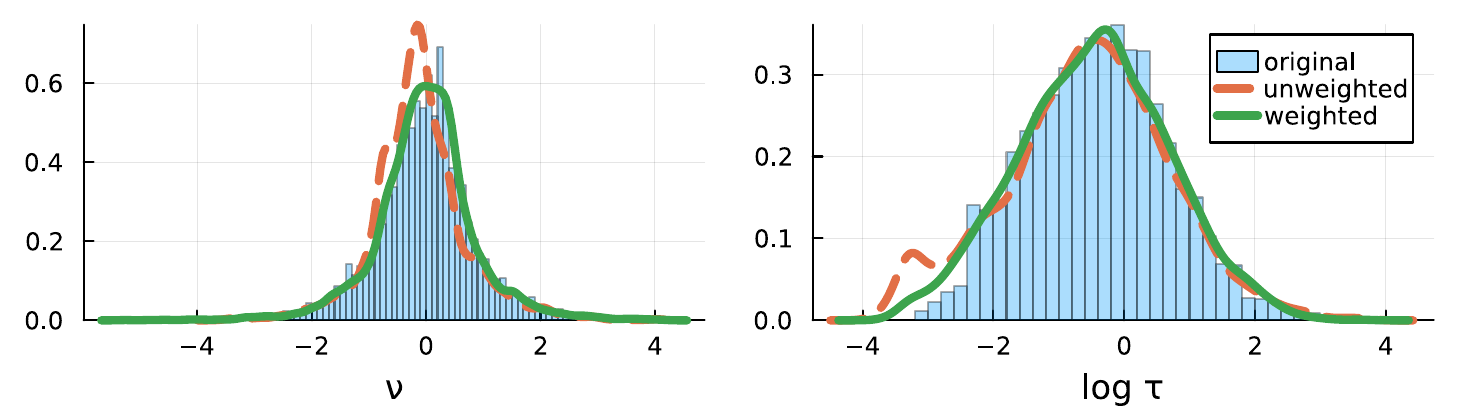}
	\caption{Normal model with hyperprior: posteriors.}
	\label{fig:hypernormal-wvo-posterior}
\end{figure}
Virtual observations are drawn from posterior predictive
distribution, but conditioning on them yields a different
posterior (dashed red).  Conditioning on a set of
weighted virtual observations with the weights found through maximization
of~\eqref{eqn:multi-level-K=1} (solid green) yields the same
posterior as the original one.
\vspace{\baselineskip}

Theorems~\ref{thm:single-level}
and~\ref{thm:multi-level-K=1} serve as building blocks of a
solution to Problem~\ref{pbm:multi-level} for $K > 1$. The probability of
$\pmb{y}^*_1, \dots, \pmb{y}^*_{K^*}$ given $x$
has the form 
\begin{equation}
\begin{aligned}
p(\pmb{y}^*_1,, \dots,, \pmb{y}^*_{K^*}|x) &=  \prod_{k=1}^{K^*} p(\pmb{y}^*_k|x) = \prod_{k=1}^{K*} \int_{\mathrm{dom}(z_k)} p(\pmb{y}^*_k|z_k)p(z_k|x) dz_k
\end{aligned}
\label{eqn:p-y-star-given-x-multi-level}
\end{equation}
\begin{dfn}
Following the form of \eqref{eqn:p-y-star-given-x-multi-level}, we define the probability
of $\pmb{v}, \pmb{w}_1, \dots, \pmb{w}_{\hat K}$ given $\pmb{x}^*$ for Problem~\ref{pbm:multi-level} as
\begin{equation}
p(\pmb{v}, \pmb{w}_1, \dots, \pmb{w}_{\hat K}) = h(\pmb{v}, \pmb{w}_1, \dots, \pmb{w}_{\hat K}) \prod_{k=1}^{\hat K} v_k \sum_{i=1}^{\hat M} w_{ki} p(\hat z_{ki}|x).
\label{eqn:p-w-given-x-multi-level}
\end{equation}
where $h(\pmb{v}, \pmb{w}_1, \dots, \pmb{w}_{\hat K})$ is a base measure over $\pmb{v}, \pmb{w}_1, \dots, \pmb{w}_{\hat K}$.
\label{dfn:p-w-given-x-multi-level}
\end{dfn}
One can view~\eqref{eqn:p-w-given-x-multi-level} as a composition of~\eqref{eqn:p-w-given-x-single-level} and~\eqref{eqn:p-w-given-x-multi-level-K=1}. Corollary~\ref{crl:multi-level} gives a solution for Problem~\ref{pbm:multi-level}:
\begin{crl}
\label{crl:multi-level}
Under the settings of Problem~\ref{pbm:multi-level},
$\pmb{v}, \pmb{w}_1, \dots, \pmb{w}_{\hat K}$ are Monte Carlo approximated by a solution of
\eqref{eqn:multi-level}:
\begin{equation}
\begin{aligned}
\pmb{v}, \pmb{w}_1, \dots, \pmb{w}_{\hat K} \approx & \arg \max \frac 1 S \sum_{s=1}^S   \sum_{k=1}^{\hat K} v_k \log \sum_{i=1}^{\hat M} w_{ki} p(\hat z_{ki}|x_s^*)  \\
       &- \log \sum_{s=1}^S \exp\left(\sum_{k=1}^{\hat K} v_k \log \sum_{i=1}^{\hat M}  w_i p(\hat z_i|x_s^*) - \sum_{k=1}^{N_k} \sum_{i=1}^{N_k^*} \log p(y_{ki}^*|x_s^*)\right) \\
&\mbox{s.t. } \sum_{k=1}^{\hat K} v_k = K^*,\,\sum_{i=1}^{\hat M} w_{ki} = 1
\end{aligned}
\label{eqn:multi-level}
\end{equation}
The approximation tends to the exact solution as $S \to \infty$.
\end{crl}
\begin{proof}The proof follows the proof of Theorem~\ref{thm:single-level} with~\eqref{eqn:p-w-given-x-multi-level} substituted instead of~\eqref{eqn:p-w-given-x-single-level}.
\end{proof}

As with~\eqref{eqn:multi-level-K=1}, $p(\pmb{y}_k^*|x)$
can be estimated by Monte Carlo integration over $z_k$, and
Algorithm~\ref{alg:multi-level} can be used to find
$\pmb{v}, \pmb{w}_1, \dots, \pmb{w}_{\hat K}$:

\begin{algorithm}
\caption{Multi-level problem.}
\label{alg:multi-level}
\begin{enumerate}
\item Perform inference on the model of
Problem~\ref{pbm:multi-level} given $\pmb{y}^*_1, \dots, \pmb{y}^*_{K^*}$ to
obtain $\pmb{x}^*$.
\item Estimate $p(\pmb{y}^*_k|x^*_s)$  for each $s \in 1 \dots
S$ and $k \in 1 \dots K$.
\item Guess $\hat{\pmb{z}}_1, \dots, \hat{\pmb{z}}_{\hat K}$, e.g.
through sampling from the inferred posterior predictive
distribution.
\item Find $\pmb{v}, \pmb{w}_1, \dots, \pmb{w}_{\hat K}$ by
maximizing~\eqref{eqn:multi-level}.
\end{enumerate}
\end{algorithm}

In the next session (Section~\ref{sec:studies}), we apply
Algorithm~\ref{alg:multi-level} to several multi-level models borrowed from
statistical literature and explore properties of the reconstructed posteriors,
in comparison to both the original posteriors and to approximations by means of
empirical Bayes.

\section{Case Studies}
\label{sec:studies}

It may be compelling to evaluate inference algorithms on large
highly-dimensional datasets. However, large datasets often lead
to convergence at the mode, which hinders such advantages of
Bayesian inference as uncertainty quantification and
incorporation of prior knowledge.  Small to medium-sized
datasets challenge algorithms to capture diverse patterns,
fostering a nuanced understanding of model robustness,
computational efficiency, and the ability to quantify
uncertainty.  For the case studies, we selected three inference
problems of gradually increasing complexity borrowed from
statistical literature. Each problem helps at highlighting a
different aspect of our approach to incremental belief updating
in probabilistic programming. We encourage the reader to consult
the supplied notebooks for details.

\subsection{Parallel Experiments in Eight Schools}
\label{sec:8schools}

This study~\citep{R81} is a popular example of application of
Bayesian multi-level modelling. Effects of coaching programs on
SAT-V scores in eight high schools are analyzed, given the
estimated treatment effect $y_i$ and the standard error of the
effect estimate $\sigma_i$ in each school.  Following
\citet[Section 5.5]{GCS+13}, who offer a detailed Bayesian
treatment of the problem, we use
Model~\eqref{eqn:8schools-model} for analysis:
\begin{equation}
\begin{aligned}
\mu & \sim \mathrm{Normal}(0, 5) \\
\tau & \sim \mathrm{HalfCauchy}(0, 5) \\
\nu_i & \sim \mathrm{Normal}(\mu, \tau) \\
y_i & \sim \mathrm{Normal}(\nu_i, \sigma_i)
\end{aligned}
\label{eqn:8schools-model}
\end{equation}
Here, $\mu$ and $\tau$ are hyperparameters, $\nu_i$ are group
parameters, and $y_i$ are observations; $\sigma_i$ is treated
as known observation noise.

First, we apply Algorithm~\ref{alg:multi-level} to reconstruct
the posterior using weighted virtual observations with $\hat M=10$.
The original and the reconstructed posteriors are shown in
Figure~\ref{fig:8schools-posteriors-and-weights}.
\begin{figure}
	\centering
	\begin{subfigure}[c]{0.69\linewidth} 
		\centering
		\includegraphics[width=0.9 \linewidth]{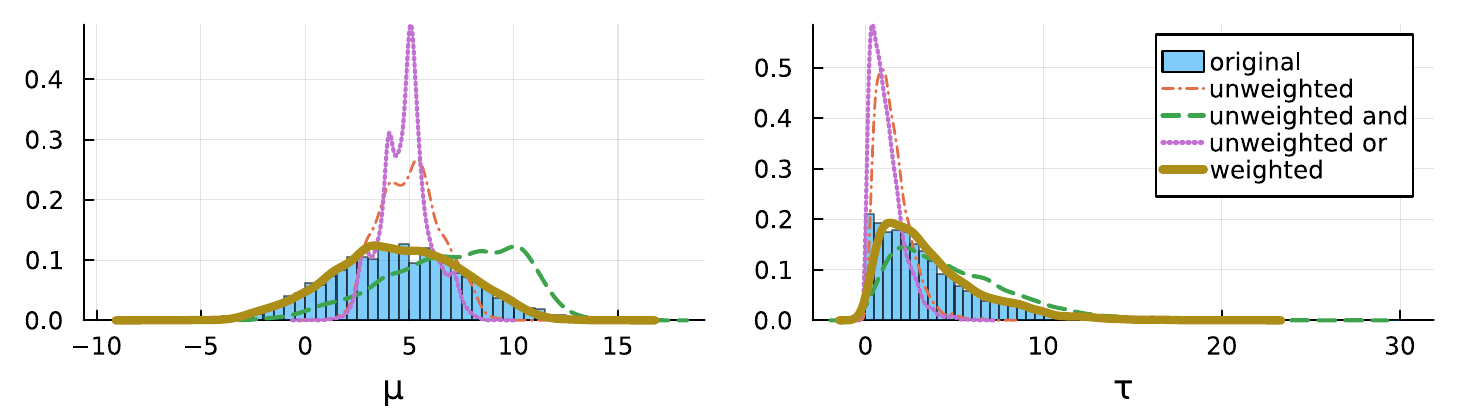}
		\caption{Hyperparameter posteriors} 
		\label{fig:8schools-wvo-posterior}
	\end{subfigure} 
	\begin{subfigure}[c]{0.69\linewidth} 
		\centering
		\includegraphics[width=0.9 \linewidth]{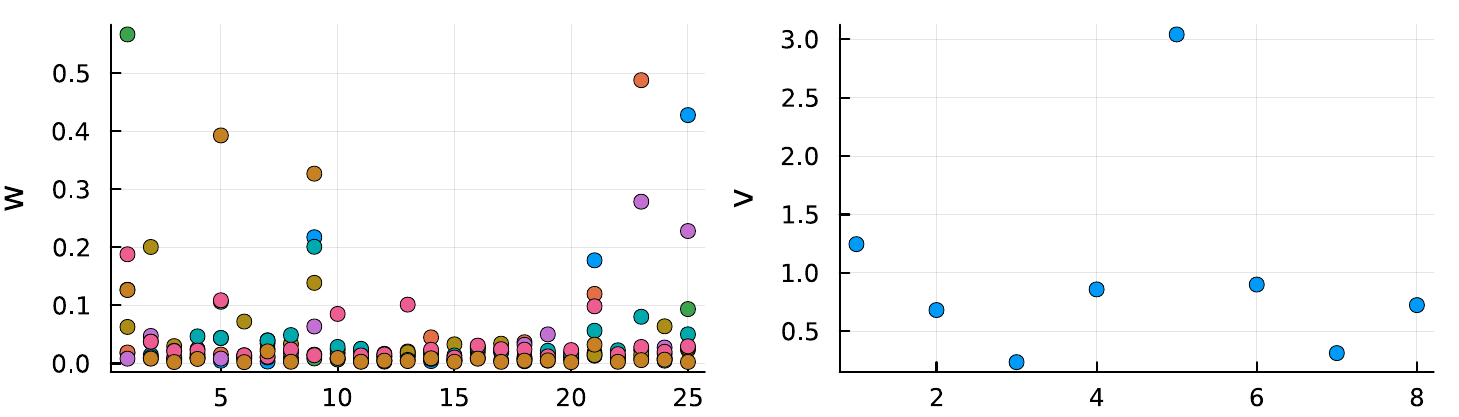}
		\caption{Inferred weights}
		\label{fig:8schools-weights} 
	\end{subfigure} 
	\caption{Eight schools: posteriors and weights.}
	\label{fig:8schools-posteriors-and-weights}
\end{figure}
As before, conditioning on weighted virtual observations results
in the same posterior as conditioning on the original dataset.
Conditioning on unweighted observations sampled from the
original posterior results in a different posterior.

Then, to simulate incremental inference, we perform
a form of leave-one-out cross-validation, as follows.
\begin{enumerate}
\item For each school, we infer the hyperparameter posterior
conditioned on the remaining seven schools.
\item Then, we construct weighted virtual observations corresponding
to the inferred posterior.
\item Finally, we infer the hyperparameter posterior conditioned
on weighted virtual observations and on the remaining school.
This posterior should match the posterior inferred on all eight
schools.
\end{enumerate}
In Figure~\ref{fig:8schools-incremental-posteriors}, each line
corresponds to leaving out a different school. While posteriors
of $\nu_i$ are of course different, the posteriors of
hyperparameters $\mu$ and $\tau$ coincide for all folds, as
expected.
\begin{figure}
\end{figure}

Finally, we compare the original posteriors of $\nu_i$ and the
incrementally inferred posteriors through weighted virtual
observations (WVO) and through marginal empirical Bayes (MEB),
with the normal and Gamma distributions fit to the
sample-approximated
posteriors of $\mu$ and $\tau$, correspondingly. Out of three
boxes of each color in
Figure~\ref{fig:8schools-original-vs-incremental-vs-meb}, the
leftmost box corresponds to the original posterior of
$\nu_i$, the middle box --- to the incrementally inferred
posterior using WVO, and the rightmost
box --- using MEB. One can observe that the
middle boxes (WVO) are better
approximations of the leftmost boxes (original) than
the rightmost boxes (MEB). Numerically, the hyperparameter posteriors
should be as close to each other as possible; the higher the
variation between folds, the worse. The standard
deviation of the mean of $\tau$ between folds is $0.20$ for
WVO, $0.24$ for MEB. For $\tau$, the standard deviation is
$0.25$ for WVO, $0.31$ for MEB.
\begin{figure}
	\begin{subfigure}[c]{0.49\linewidth} 
		\centering
		\includegraphics[width=0.9\linewidth]{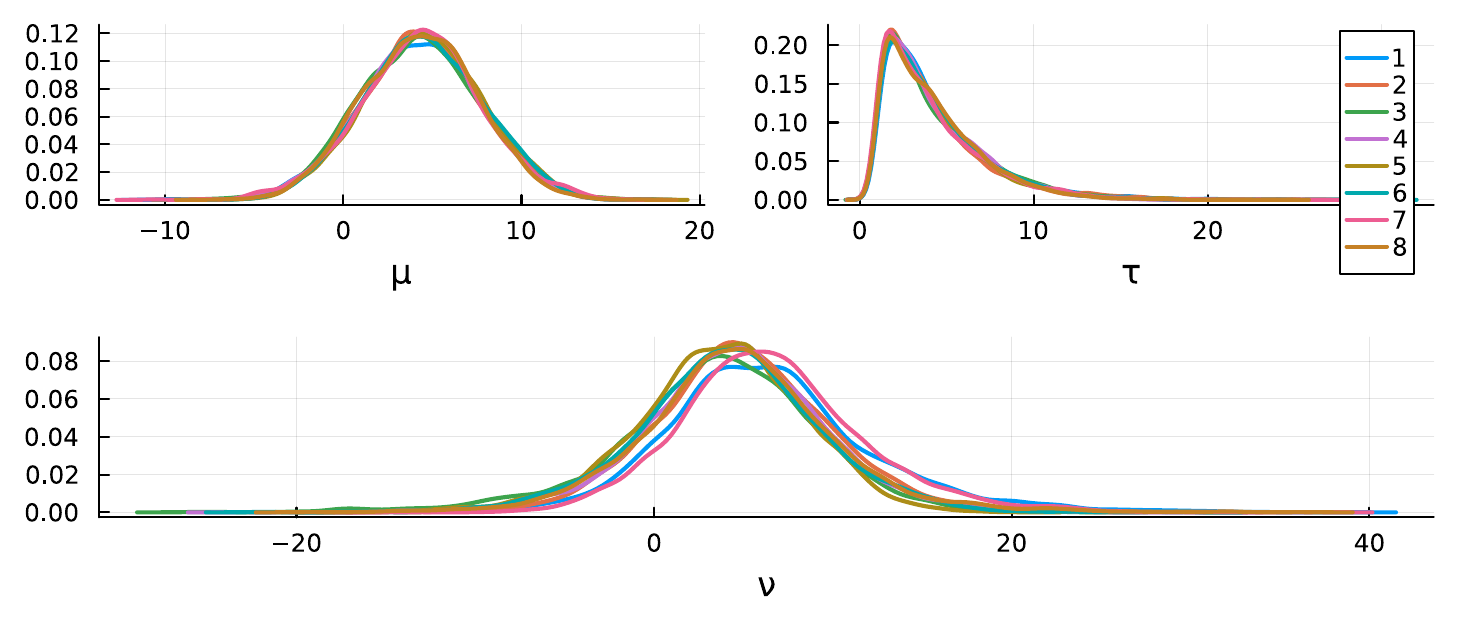}
		\caption{Incremental posteriors}
		\label{fig:8schools-incremental-posteriors}
	\end{subfigure} 
	\begin{subfigure}[c]{0.49\linewidth} 
		\centering
		\includegraphics[width=0.9\linewidth]{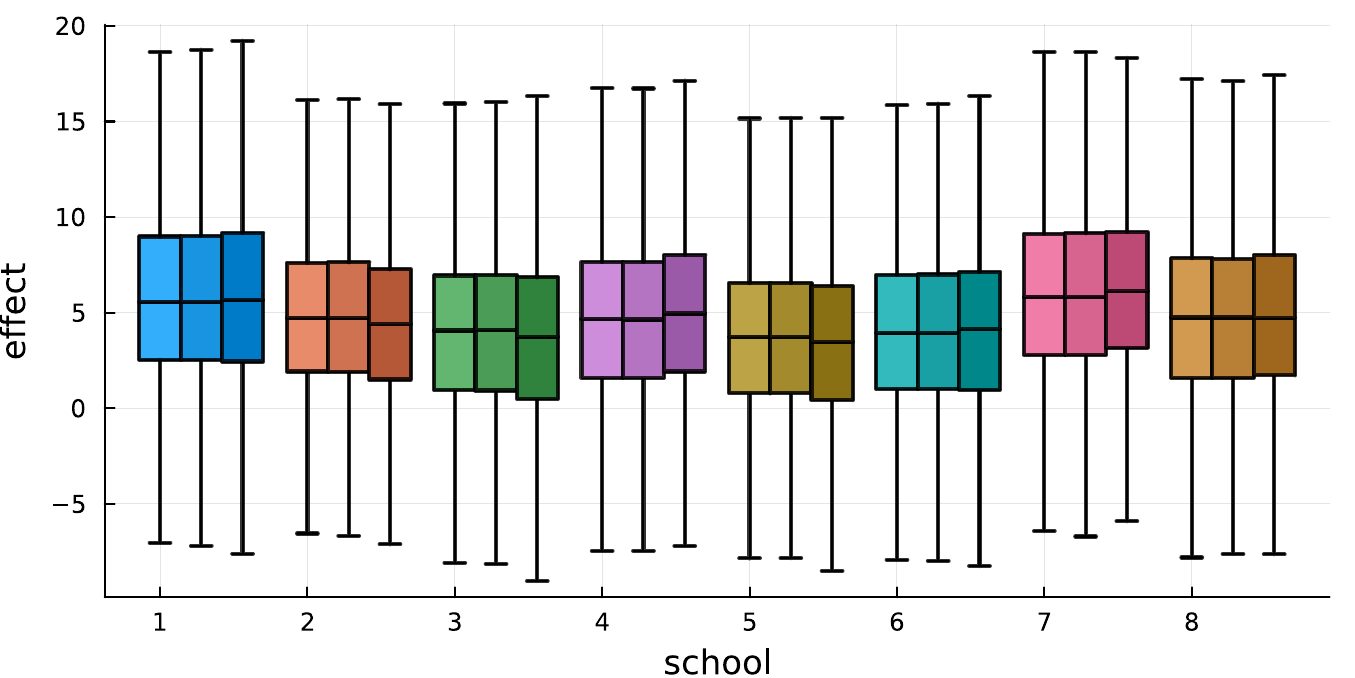}
		\caption{Original vs. WVO vs. MEB}
		\label{fig:8schools-original-vs-incremental-vs-meb}
	\end{subfigure} 
	\caption{Eight schools: cross-validation.}
	\label{fig:8schools-cross-validation}
\end{figure}

\subsection{Tumor Incidence in Rats}
\label{sec:rats}

The dataset on tumor incidence rats in 71 labs was introduced
and initially analyzed by \citet{T82}. For each lab, the total
number of rats in the experiment $n_i$, and the number of rats which
developed cancer $y_i$ are given. Again, following~\cite{GCS+13}, we
use Model~\eqref{eqn:rats} in our analysis:
\begin{equation}
\begin{aligned}
	p(\alpha, \beta) & \propto (\alpha + \beta)^{-5/2} \\
	\eta_i & \sim \mathrm{Beta}(\alpha, \beta) \\
	y_i & \sim \mathrm{Binomial}(n_i, \eta_i)
\end{aligned}
\label{eqn:rats}
\end{equation}
Here, $\alpha$ and $\beta$ are hyperparameters, $\eta_i$ are
group parameters, and $y_i$ are observations. 

Just as before, we apply Algorithm~\ref{alg:multi-level} with
$\hat M=10$ to reconstruct the posterior. The original and the
reconstructed hyperparameter posteriors are shown in
Figure~\ref{fig:rats-wvo-posterior}. Then, we evaluate
incremental inference by first inferring and reconstructing the
posterior on each set of 70 labs, leaving one lab out, and then
inferring the remaining lab's posterior by conditioning on the
weighted virtual observations and that lab's data.
\begin{figure}
	\centering
	\includegraphics[width=0.62\linewidth]{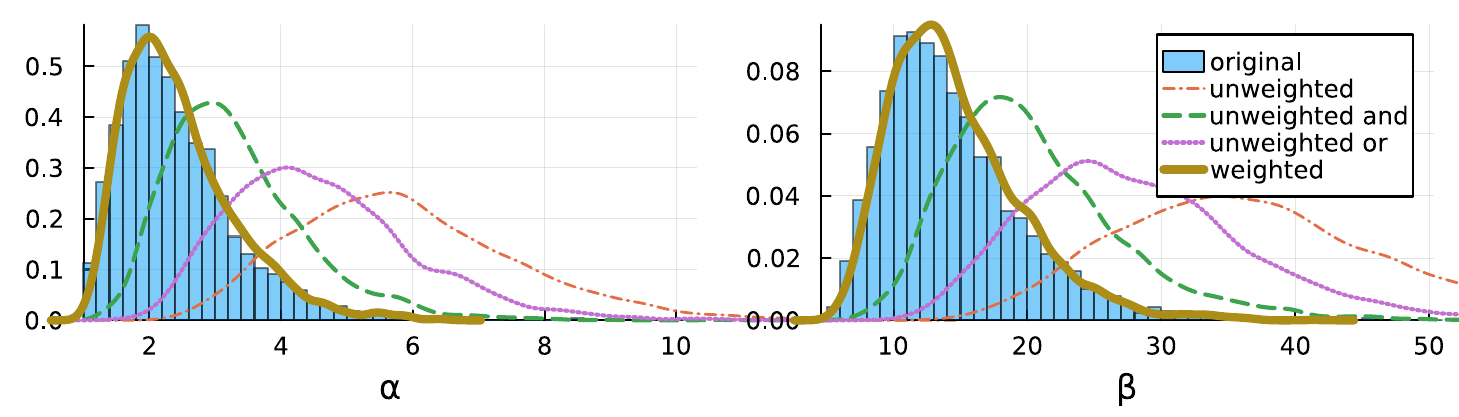}
	\caption{Rats: hyperparameter posteriors.}
	\label{fig:rats-wvo-posterior}
\end{figure} 
In Figure~\ref{fig:rats-original-vs-incremental}, we graphically
compare the original and the incrementally inferred parameter
posteriors by plotting ellipses with the center in the mean of
the original (horizontal axis) and the incremental (vertical
axis) posterior with the standard deviations of the original and
the incremental posterior as the width and the height. This
visualization suits better than a bar plot for a large number of
groups.  Ideally, all labs should be represented by perfect
circles with their centers lying on the bisector of the
coordinate axes; the results are quite close to that.

Additionally, we use this case study to explore the influence of
subsampling of groups ($\hat K < K^*$) on the reconstruction
quality. Figure~\ref{fig:rats-approximation-error} shows the
standard deviations of the means of $\alpha$ and $\beta$ among the
cross-validation folds as functions of $\hat K$.
According to the plots, for $\hat K \le 20$ the subsampling leads
to an increased reconstruction error. However, for $\hat K > 20$
the error is the virtually the same as for $\hat K = K^*$, an
evidence that for multi-level data sets with a large number of groups,
virtual observation sets with $\hat K \ll K^*$ are likely to be
sufficient.

\begin{figure}
\centering
\begin{subfigure}[c]{0.44\linewidth}
	\centering
	\includegraphics[width=0.85\linewidth]{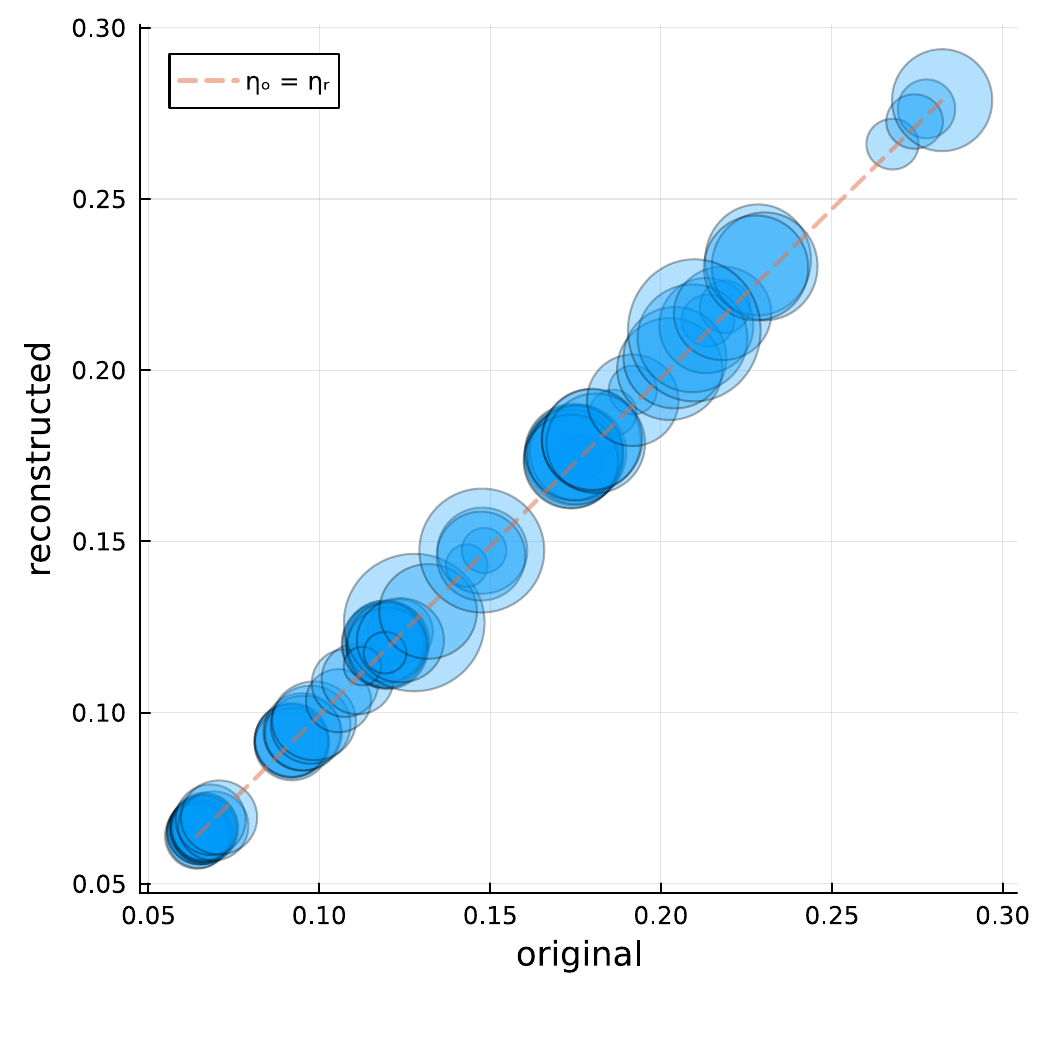}
	\caption{Incremental vs. original.}
	\label{fig:rats-original-vs-incremental}
\end{subfigure} 
\begin{subfigure}[c]{0.54\linewidth}
	\centering
	\includegraphics[width=0.85\linewidth]{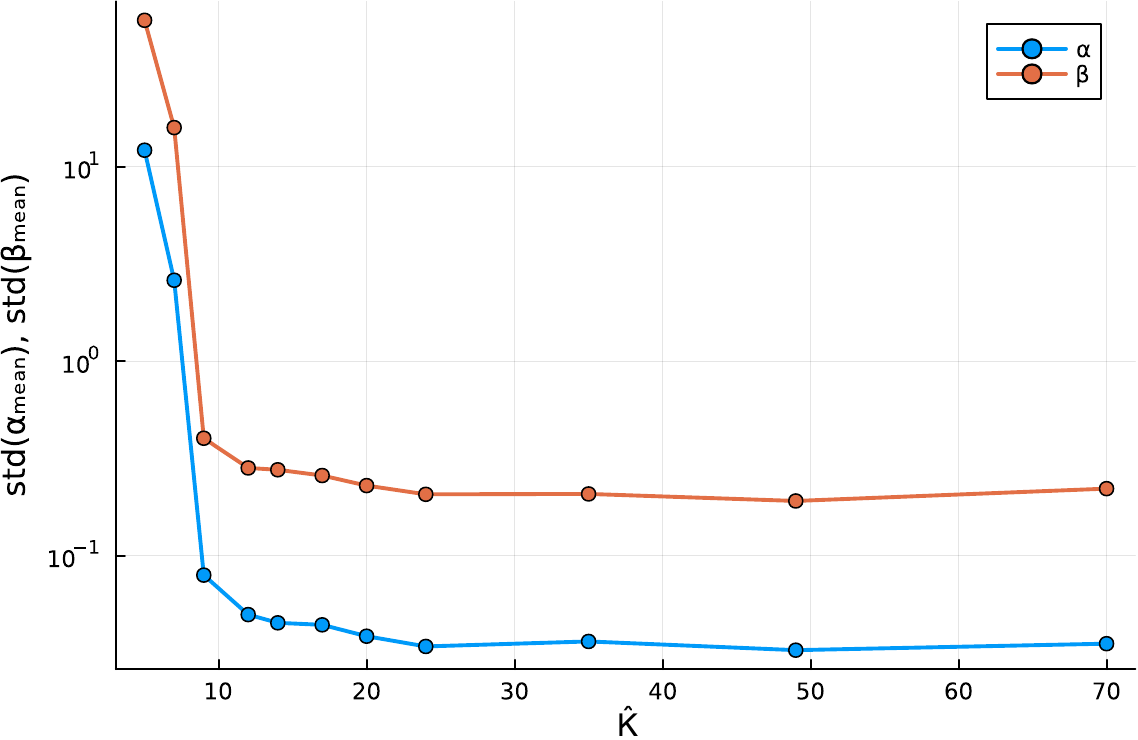}
	\caption{Subsampling of groups.}
	\label{fig:rats-approximation-error}
\end{subfigure}
\caption{Rats: incremental posteriors.}
\end{figure}

\subsection{Educational Attainment in Secondary Schools}

We conclude our case studies with an analysis of a 
data are on 3435 children who attended 19 secondary schools in
Scotland~\citep{P91}. There are 11 fields in the original data
set of which the following are used:
\begin{itemize}
\item SID --- secondary school identifying code.
\item SEX --- pupil's gender.
\item SC --- pupil's social class scale.
\item VRQ --- a verbal reasoning score on entry to the secondary
school.
\item ATTAIN --- pupil's attainment score at age 16.
\end{itemize}

We fit the following mixed effects model: 
\begin{equation}
\begin{aligned}
\mu_{0\beta}, \mu_{0\zeta} & \sim \mathrm{Normal}(0, 5) \\
\sigma_{0\beta}, \sigma_{0\zeta} & \sim \mathrm{Cauchy}(0, 5) \\
\beta & \sim \mathrm{Normal(\mu_{0\beta}, \sigma_{0\beta})} \\
\zeta^2 & \sim \mathrm{LogNormal}(\mu_{0\zeta}, \sigma_{0\zeta}) \\
y & \sim  \mathrm{Normal}(\beta \oplus u, \sqrt{\zeta^2 \oplus v})
\end{aligned}
\label{eqn:attain}
\end{equation}
Here, $\mu_{0,\beta} \in \mathcal{R}^4, \mu_{0, \zeta} \in
\mathcal{R}, \sigma_{0, \beta} \in \mathcal{R}^4,
\sigma_{0, \zeta} \in \mathcal{R}$ are the
hyperparameters. $\beta \in \mathcal{R}^{19 \times 4}$ and
$\zeta^2 \in \mathcal{R}^{19 \times 1}$ are the regression
coefficients for the mean and the variance of the conditional
distributions of observations $y$.  $u$ and $v$ are explanatory
variables. We omit tensor indices to avoid clutter.

This is both a much larger dataset and more elaborated model
than in the two previous studies. The purpose of this case study
is to demonstrate that our approach can handle larger data,
regression, and mixed effects. We apply
Algorithm~\ref{alg:multi-level} to Model~\eqref{eqn:attain} with
$\hat M=25$ and $\hat K = K^* = 19$. The original and the
reconstructed hyperparameter posteriors are shown in
Figure~\ref{fig:attain-wvo-posterior}. The reconstructed
marginal posteriors are fairly close to the original ones.
\begin{figure}
	\centering
	\includegraphics[width=0.6\linewidth]{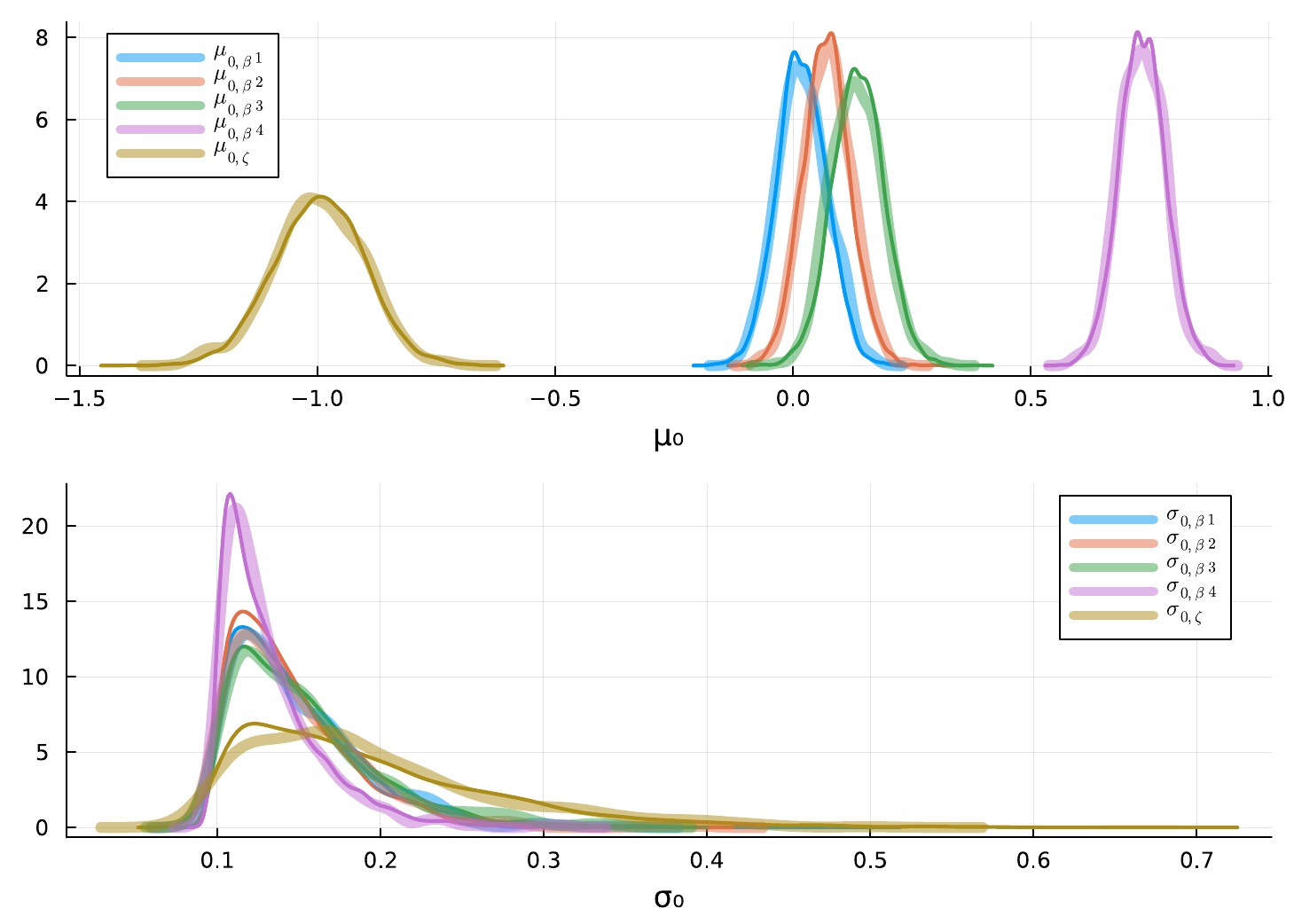}
	\caption{Attainment: hyperparameter posteriors: thin opaque
	curves are the original posteriors, broad semitransparent
	curves are the reconstructed ones.}
	\label{fig:attain-wvo-posterior}
\end{figure}
This case study also lets us illustrate both data compression and
inference speedup due to the use of weighted virtual
observations. First, we replace $\mathbf{3435}$ original observations with
$\hat M \cdot \hat K = \mathbf{475}$ weighted virtual observations, a
\textbf{6-fold} decrease in the number of observations. In addition, for
incremental inference on a new secondary school, we get a
$\mathbf{15}$-dimensional model instead of $\mathbf{105}$-dimensional model on
the original observations. As a result, the inference on the
original model takes $\approx \mathbf{1830}$ seconds on our computer, while the
inference on a single school and weighted virtual observations
reconstructing the posterior on the rest of the schools only
takes $\approx \mathbf{240}$ seconds, an almost \textbf{8-fold}
decrease in the running time.

\section{Related Work}
\label{sec:related}

The importance of learning from data is well
appreciated in probabilistic programming. Along with empirical
Bayes, applicable to probabilistic programming as well as to
Bayesian generative models in general, probabilistic-programming
specific approaches were proposed. One possibility is to induce
a probabilistic program suited for a particular data
set~\citep{LJK10, PW14, P18, HSG11}.  A related but different
research direction is inference compilation~\citep{LBW17,BSB+19},
where the cost of inference is amortized through learning
proposal distributions from data. Another line of research is
concerned by speeding up inference algorithms by tuning them
based on training data~\citep{ETK14,MSH+18}.  Our approach to
learning from data in probabilistic programs is different in
that it does not require any particular implementation of
probabilistic programming to be used, nor introspection into the
structure of probabilistic programs or inference algorithms.

Approximation of a large sample set by a small weighted subset
bears similarity to Bayesian
coresets~\citep{HCB16,CB19,MXM+20,ZKK+21} --- a family of
approaches aiming at speeding up inference with large datasets.
A Bayesian coreset is a small weighted subset of the original
large dataset, with the promise that inference on the coreset
yields the same or approximately the same posterior. However,
there are significant differences between Bayesian coresets and
the setting in this work. First, in Bayesian coresets, the
posterior is unknown when the coreset is constructed. In this
work, the posterior distribution of the hyperparameter is known
(as  a Monte Carlo approximation) before the samples are
selected and the weights are computed.  In particular,
\cite{CB19} minimize the KL divergence between the approximate
and the exact posterior, while this work minimizes the
complementary KL divergence between the exact posterior and its
approximation.  This facilitates a simpler formulation of
divergence minimization. Second, careful selection of samples to
be included in the coreset is necessitated by high
dimensionality of data in the dataset~\citep{MXM+20}. However,
even in elaborated multi-level Bayesian models the group
parameters are low-dimensional, and interdependencies between
multiple hierarchies in cross-classified models are assumed to
be negligible. Because of that, a random draw from the parameter
posterior suffices even if it might not work in a
higher-dimensional setting. 

In the context of privacy protection in distributed computing,
this work is related to federated learning~\cite{WZY+22}, and to
server-side Bayesian federated learning~\citep[Section
3.3]{CCF+23} in particular.  Federated learning aims at training
a machine learning model, with a focus on neural networks, on
multiple datasets contained on different nodes, exchanging model
parameters rather than data samples. Bayesian federated learning
incorporates principles of Bayesian learning into federated
learning framework.  Server-side Bayesian federated learning is
concerned with aggregation of the local posteriors into the
global posterior, and decomposition of the global posterior to
serve as the priors in the local models. The theoretical
foundation and the algorithm introduced in this work may serve
as the basis for a novel approach to the aggregation and
decomposition of the posterior in Bayesian federated learning,
in particular when local models use Markov chain Monte Carlo
methods for inference.

\section{Conclusion}
\label{sec:conclusion}

In this work, we addressed the challenge of efficient
incremental belief updates in probabilistic programming,
particularly when dealing with complex Bayesian statistical
models. Our approach diverges from the commonly employed
empirical Bayes method, which relies on parametric
approximations of the posterior distribution. Instead, we
propose a novel theoretically grounded strategy based on
conditioning the model on a set of weighted virtual
observations. 

We presented a novel algorithm addressing the
challenge of efficient incremental belief updates through the
use of weighted virtual observations. The strength of our
algorithm lies in its ability to seamlessly integrate with
probabilistic programming frameworks, allowing for efficient
belief updates in the face of incremental observations. We
provided a reference implementation of our algorithm,
demonstrating its effectiveness and robustness through a series
of didactic examples and case studies. 

However, it's essential to recognize that our algorithm serves
as a starting point and can be further enhanced through
variations in several key aspects. Future research directions
could explore altering the selection process for virtual
observations, determining the optimal number of virtual
observations, and refining the search strategy for observation
weights. These avenues provide exciting opportunities for
researchers to build upon the foundation laid by our algorithm
and tailor it to specific application domains.

\section*{Acknowledgements}

This research is partially supported by the Israeli Science Foundation grant
1156/22.

\bibliographystyle{plainnat}
\bibliography{refs}

\end{document}